\documentclass[twoside,11pt]{article}

\usepackage{blindtext}
\usepackage{amsmath}
\usepackage{siunitx}
\usepackage{gensymb}
\usepackage{amsthm}
\usepackage{multirow}
\usepackage{graphicx}
\usepackage{hyperref}
\usepackage[numbers]{natbib}
\newtheorem{theorem}{theorem}
\newtheorem{conjecture}{Conjecture}
\newtheorem{definition}{Definition}

\usepackage{lastpage}

\begin{document}

\title{On the Development of Binary Classification Algorithm Based on Principles of Geometry and Statistical Inference}

\author{ Vatsal Srivastava \\ 
        Department of Computational Intelligence\\
       School of Computing\\
       SRM Institute of Science and Technology\\
       Chennai, Tamil Nadu, India}

\maketitle

\begin{abstract}%   <- trailing '%' for backward compatibility of .sty file
The aim of this paper is to investigate an attempt to build a binary classification algorithm using principles of geometry such as vectors, planes, and vector algebra. The basic idea behind the proposed algorithm is that a hyperplane can be used to completely separate a given set of data points mapped to n dimensional space, if the given data points are linearly separable in the n dimensions. Since points are the foundational elements of any geometrical construct, by manipulating the position of points used for the construction of a given hyperplane, the position of the hyperplane itself can be manipulated. The paper includes testing data against other classifiers on a variety of standard machine learning datasets. With a focus on support vector machines, since they and our proposed classifier use the same geometrical construct of hyperplane, and the versatility of SVMs make them a good bench mark for comparison. Since the algorithm focuses on moving the points through the hyperspace to which the dataset has been mapped, it has been dubbed as moving points algorithm. 
\end{abstract}

\begin{center}
\end{center}

\section{Introduction}
The first mathematical model of a neuron was described by Warren McCulloch and Walter Pitts in 1943. This was followed by Frank Rosenblatt's implementation of a perceptron (\cite{rosenblatt_perceptron_1957}). It was further improved upon by Prof. Bernard Widrow and his student Ted Hoff with the introduction of adaptive linear neuron (\cite{widrow_adaptive_nodate}). This was also the time when a geometrical model k-nearest neighbors was introduced. Some of these models are similar or iterations of their predecessors(for example, adaline can be considered as an upgrade to the perceptron model), while others explore different fields of mathematics, such as KNN which has already explored the area of geometry for machine learning. Clustering is another example which uses a geometrical approach for classifications. Then there are powerful classifiers like support vector machines (\cite{Cortes1995SupportVectorN}) which are a giant leap in computational efficiency and accuracy. After an extensive study of existing classifiers, it seemed reasonable to attempt to build a classifier which combines the geometrical aspect and iterative nature of the aforementioned classifiers to create a new classifier with, of course, an acceptable accuracy. 

\section{Problem Statement}
The basic idea is to combine the geometrical application of machine learning an iterative approach. The problem is that a proposed algorithm needs to have a mathematical basis, it needs to be sufficiently distinct from existing algorithms, and lastly, it needs to be tested and proved to have an acceptable level of accuracy.

Hence, in this paper we shall look to answer the following three questions: 

\begin{enumerate}
    \item Is the proposed algorithm mathematically and logically sound?
    \item Is the proposed algorithm different from existing algorithms and if there are similar approaches, then what are the overlapping areas.
    \item How does the algorithm perform on different standard and popular datasets? 
\end{enumerate}

To begin with answering these questions, we shall first take a look at how the algorithm proceeds with initialization, training and prediction.

\section{Moving Points Algorithm}

To lay the foundation of a new classifier, let us first establish some conjectures for the cartesian plane;
\begin{conjecture}
Let there be two sets of points $X_{0}$ and $X_{1}$, such that there exist no two line segments that can be drawn from a set of four distinct points with two and only two points from one set, that will intersect unless the two endpoints of one segment belong to different sets, then:    
 \begin{enumerate}
     \item There will exist at least one line $l_1$that will divide the plane into two (excluding the set of points on the line) disjoint sets $A$ and $B$ such that $X_{0} \subset A$ and $X_{1} \subset B$.
     \item No line segment drawn by joining two points of the same set from $X_{0}$ or $X_{1}$ shall meet with the line $l_1$.
     \item Any line segment drawn by joining two points from the different sets $X_0$ and $X_1$ shall always meet the line $l_1$
 \end{enumerate}
    
\end{conjecture}

\subsection{Initialization Algorithm}
Since the postulate states that a line segment joining points from different sets will always meet, then we can be sure that a line segment joining the medians of the two datasets $X_0$ and $X_1$ will intersect a decision boundary. Why, is shown as follows:

    Let us consider the sets $X_0$ and $X_1$ as sample spaces $\Omega_0$ and $\Omega_1$. $X_1, X_2\dots,X_n$ $\sim$ $\delta_i$ $\forall$ $i \in [P_1, P_n$] where $P_1, P_2, P_3, \dots$ are points in $\Omega_0$. Similarly, $Y_1, Y_2\dots,Y_m$ $\sim$ $\delta_j$ $\forall$ $j \in [Q_1, Q_m$] where $Q_1, Q_2, Q_3, \dots$ are points in $\Omega_1$.

If $X_1, X_2 \dots, X_n$ are IID, then we can say that\footnote{Theorem 5.6, \cite{wasserman2010statistics}};
\[\overline{X}_{n} \xrightarrow{P} \mu_{X_0} \] 
Which can be interpreted as \textit{the distribution of $\overline{X}_{n}$ becomes more concentrated around $\mu$ as $n$ gets large.}

Since we can calculate $\overline{X}_{n}$ from our sample, we can say that the expected value of total populace will lie somewhere in proximity. Thus, by minimizing collinearity between the decision boundary $l_1$ and line segment joining $\overline{X}_{n}$ (or other measures of central tendency), we reduce collinearity between expectation $\mu_{X_0}$ and $l_1$ greatly. Trivially, the sample and the total populace would be distributed uniformly around $\mu_{X_0}$. Hence, by reducing the collinearity, we increase the accuracy of our decision boundary. 

Same can be said for $\overline{Y}_{m}$;
\[\overline{Y}_m \xrightarrow{P} \mu_{X_1}\]

The term collinear is used for classification. Since we want to "reduce" collinearity, we need to use a measure for \textit{degree of collinearity}. Such a measurement can be done using the angle that two segments(or line and segment) subtend and their proximity(distance).

\begin{definition}
Two line segments are said to be near collinear if the angle subtended by them and their proximity is below an arbitrary threshold. With them being collinear if both values are 0.
\end{definition}
To extend this definition to our line $l_1$ and line segment joining mean of sets $X_0$ and $X_1$(say $\overline{AB}$, we can consider Euclid's second postulate\cite{noauthor_euclids_1482}: \textit{Any straight line segment can be indefinitely extended in a straight line.}
Let us consider the line segment $\overline{AB}$ as $\Vec{s}$ and a line segment $\Vec{r}$ on $l_1$ such that $\Vec{r} \cap \Vec{s} \neq \phi$
Therefore the angle $\theta$ between $\Vec{s}$ and $\Vec{r}$ is given by:

\[    \theta = \cos^{-1}\bigg[\frac{\Vec{s} \cdot \Vec{r}}{|\Vec{s}|\cdot |\Vec{r}|}\bigg]\]

The maximum value of $\cos^{-1}$ is $\pi$ but we need to remember that there are two angles when line intersects. Both angles are supplementary hence to maximize both, each needs to be $\frac{\pi}{2}$.

The above section shows why a line $l_1$ passing through line segment joining the segment means and perpendicular to it is a good choice for initialization of two random points. 

\subsection{Learning Algorithm} \label{sec:sec32}
After initialization, we will have two points $E \equiv (x_1, y_1)$ and $F \equiv (x_1, y_2)$ which can be used to write equation of a lines as:
\[
    y - y_1 = \frac{y_2 - y_1}{x_2 - x_1}(x - x_1)
\]
\[
    y = \frac{y_2 - y_2}{x_2 - x_1}(x - x_1) + y_1
\]
\[
    (x_2 - x_1)y = xy_2 - x_1y_2 - y_1x + y_1x_1 + y_1x_2 - y_1x_1    
\]
\begin{equation}
    (x_2 - x_1)y = (y_2 - y_1)x - x_1y_2 + y_1x_2
\end{equation}    

compare equation (1) to $Ax + By + C$. Putting the values of the mean points (or median) we will get one positive and one negative value, which will be used to assign pseudo class $\in \{-1, 1\}$ to the sets $X_0$ and $X_1$. The sign of pseudo class should be similar to that of their respective outputs.

Then we calculate the displacement $d$ value of each point $T \equiv (x_0, y_0)$.
\begin{equation}
    d = \frac{Ax_0 + By_0 + C}{\sqrt{A^2 + B^2}}
\end{equation}
    
Using values from equation (1), we can write the equation of displacement as:
\begin{equation}
    d = \frac{(y_1 - y_2)x_0 + (x_2 - x_1)y_0 + (x_1y_2 - x_2y_1)}{\sqrt{(y_1 - y_2)^2 + (x_2 - x_1)^2}}
\end{equation}
    
If $d \times$ pseudo class $< 0$, we can infer that the classification is incorrect. Let this quantity be $\lambda$.
Thus $\lambda$ has a information about whether the classification is correct or not and also about the degree of misclassification(points further away or with higher displacement are more misclassified than the nearer ones). Hence we can move the decision boundary in an adaptive manner. 
\begin{enumerate}
    \item Upon iterating through the dataset, suppose we encounter a point $Q$ which was misclassified. Let the position vector of that point be $\Vec{d}$.
    \item Out of the two points $E$ and $F$, we will calculate which point is closer to point $Q$. Let the position vector of this point be $\Vec{c}$.
    \item We will use these two position vectors to calculate the vector joining these points;
    
        \[\Vec{u} = \Vec{c} - \Vec{d}\]
    \item Now we select a random point from the near-cluster\footnote{near cluster is a set of points at a specified distance metric from the sample mean. Random points are selected to suppress the effects of outliers in learning.} and let this point have the position vector $\Vec{g}$. Using this we will calculate another vector $\Vec{w}$ as;
    \[
        \Vec{w} = \Vec{g} - \Vec{d}
    \]
\end{enumerate}

Now we can calculate the displacement vector $\Vec{v}$ (not related to displacement value $d$) as:
\[
    \Vec{v} = \vec{w} - \Vec{u}
\]
\begin{figure}[h]
    \centering
    \includegraphics[width=0.45\textwidth]{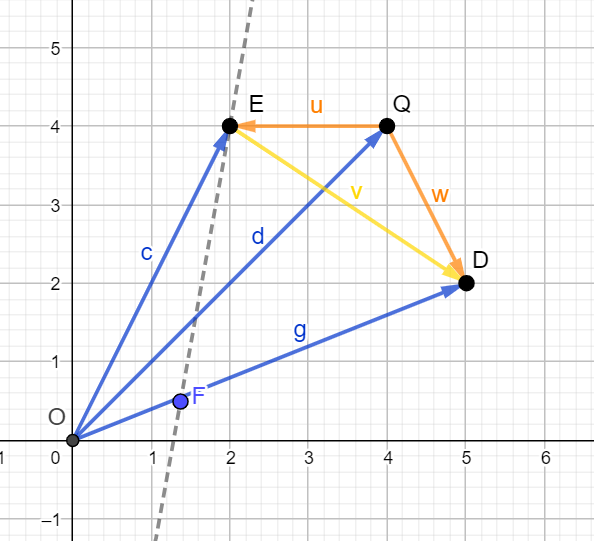}
    \caption{Graphical representation of Calculation of displacement vector}
    \label{fig:graphical_representation_of_movement}
\end{figure}

Then the unit displacement vector which tells us the \textit{direction} of movement is given by: 
\[
    \hat{v} = \frac{\vec{v}}{|\Vec{v}|}
\]
the displacement factor which tells us the magnitude of movement (or update) is given by $|\eta \times \lambda|$. 
The addition or \textit{movement vector} is given by:
\[
    \Vec{t} = \hat{v} \times |\eta \times \lambda|
\]
Finally the new position vector $\Vec{n}$ of the migrating point will be given by: 
\[
    \Vec{n} = \Vec{a} + (\hat{v} \times |\eta \times \lambda|) 
\]
The reason why this would work is again based on proximity. Conjecture 1(1) states that $X_0 \subset A$, $X_1 \subset B$ and $A \cap B = \phi$. Therefore a misclassification would imply a situation such that $Q \in X_0$ and $Q \in  B$. If the data is linearly separable then in order to correct the classification, decision boundary should move in such a way that the resulting boundary satisfies $Q \notin B$. For that to happen, the decision boundary(or its points) shall move in such a way that it moves towards the sample mean of class whose points are a subset of $B$ (for a more rigorous explanation, see Appendix A).

There are various ways to move the decision boundary in a manner that satisfies the above conditions. For example by moving points along vector joining the sample means of \textit{both class}. Or simply move the point towards the sample mean of opposite class. The reason for using the method described is that it takes into account the relative position of point that was misclassified. Therefore if an outlier is misclassified, the model would update accordingly.

\subsection{Over-fitting Cases}
Like all classifiers, the MPA can also suffer from over training. As the position of points is updated, They are displaced towards the central tendency of the class they misclassified. This also gives them a displacement towards each other since. The closer two points are, the more drastically the orientation of the line made by joining them changes with change in the relative position of the two points. Therefore if the model trains for a long time, two points will get too close to each other and the decision boundary will start swinging.

\begin{figure}[h]
    \centering
    \includegraphics[width=0.45\textwidth]{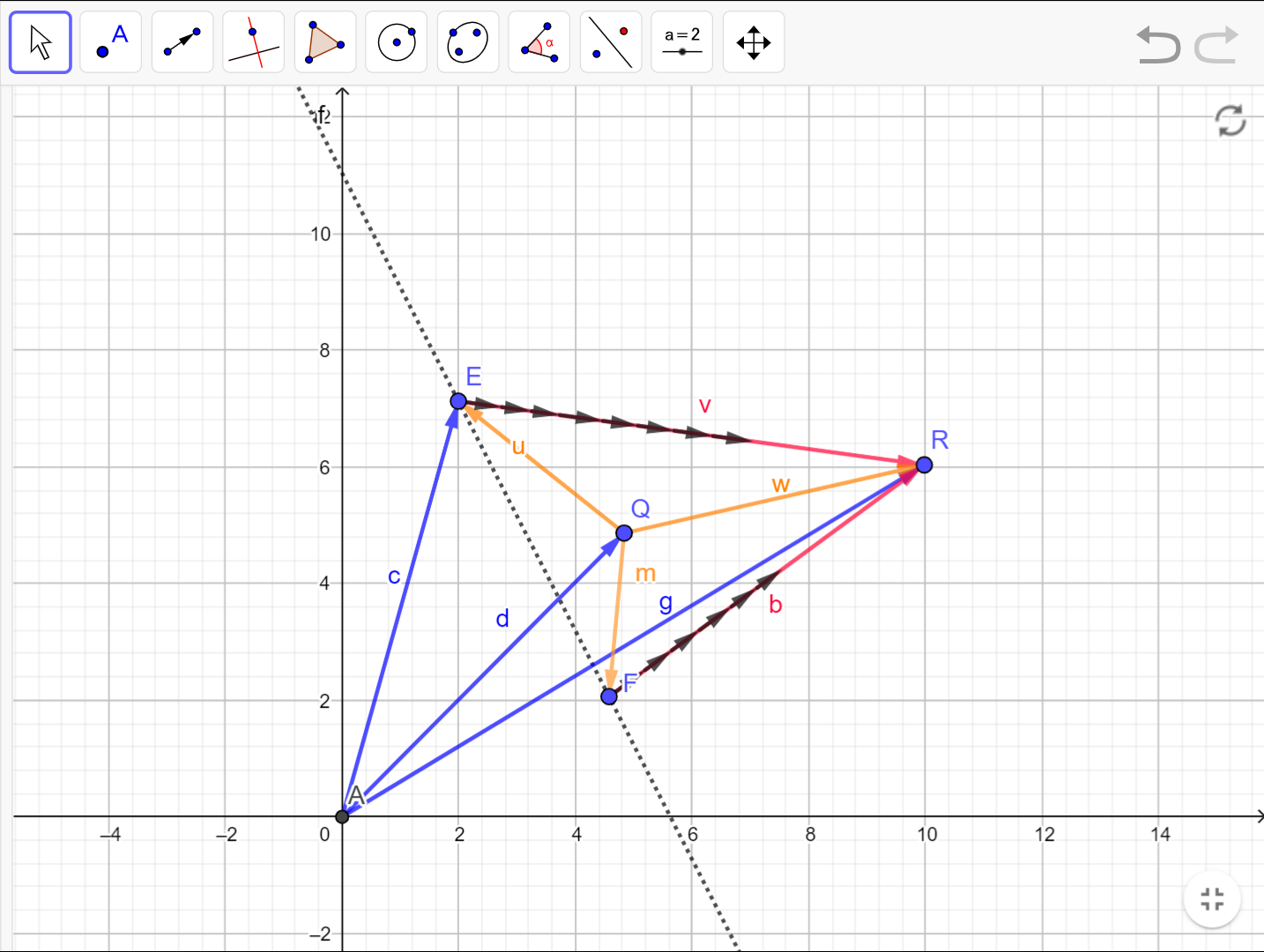}
    \caption{Two moving points start moving close to each other as the model trains.}
    \label{fig:path_mapping}
\end{figure}
The above figure shows how two migrating points $E$ and $F$ might move along the respective displacement vectors $\Vec{v}$ and $\Vec{b}$. The misclassified point is point $Q$ and $R$ is a random point in the near cluster of opposing class. It can be observed that the points get drawn closer as the model trains. 

This problem becomes even more prominent if the datasets are not linearly separable. In this case, after the decision boundary reaches an optimum position, the moving points say, $E$ and $F$ will effectively have a displacement along the line connecting the central tendencies equal to 0\footnote{Upon reaching a position in case of non-linearly separable datasets where no further improvement is possible it is probable that the model will mis-classify a near equal number of examples hence the displacement towards either class gets cancelled out by a movement towards the opposite class.}.
But, the displacement towards the line segment joining central tendencies is in the same direction. Hence the two points will converge. 
\begin{figure}[h]
    \centering
    \includegraphics[width=0.45\textwidth]{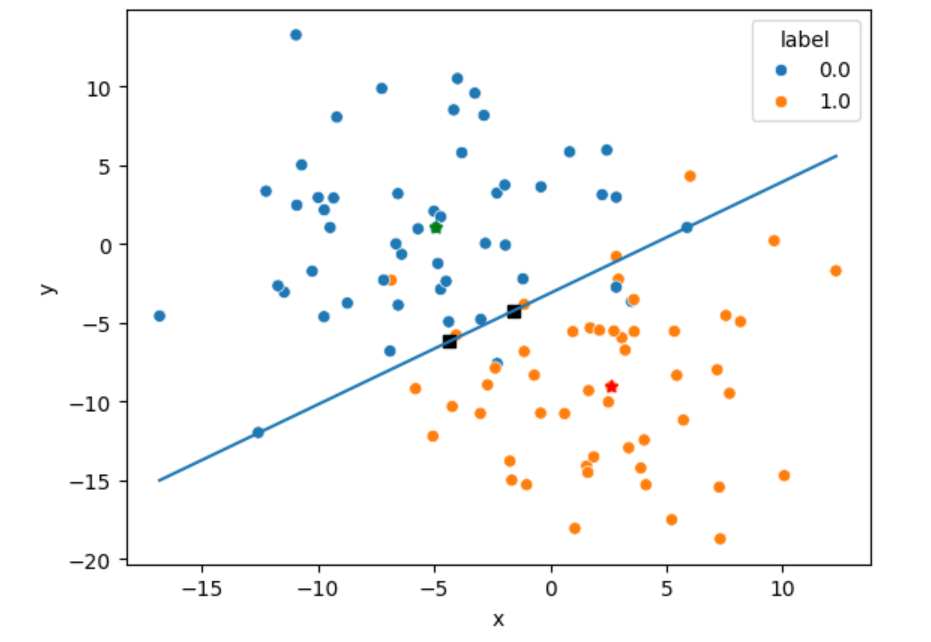}
    \caption{Model trains over a non-linearly separable dataset for 150 epochs}
    \label{fig:150_epochs_test}
\end{figure}

\begin{figure}[h]
    \centering
    \includegraphics[width=0.45\textwidth]{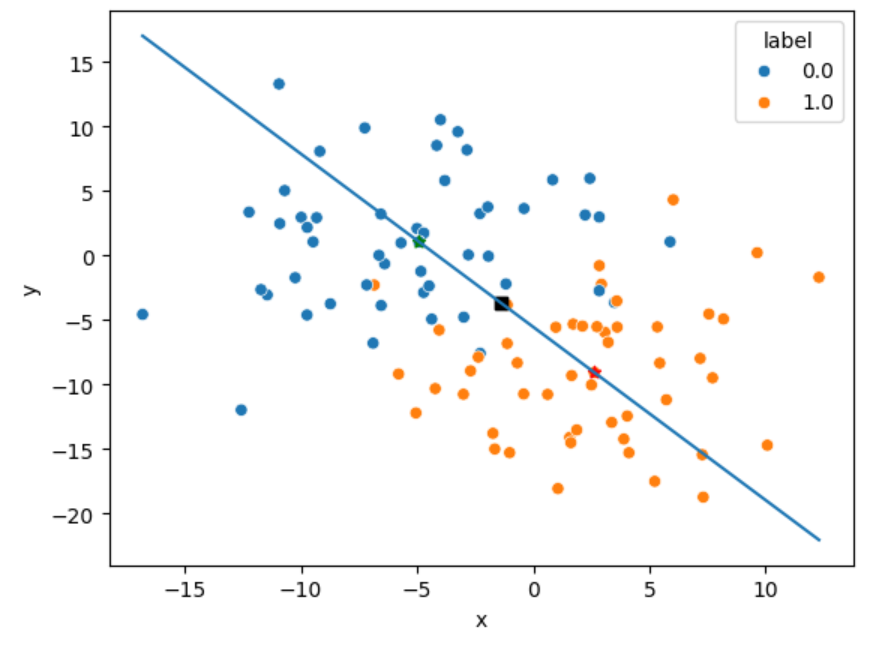}
    \caption{Same model as above trains over the same dataset for 5000 epochs}
    \label{fig:5000_epochs_test}
\end{figure}

\autoref{fig:150_epochs_test} and \autoref{fig:5000_epochs_test} compare the position of the two moving points after training over the same dataset with same initialization for 150 and 5000 epochs. As can be seen, when we kept updating the model for a high number of epochs, the points didn't move around much but they did converge at a point\footnote{This will not always be the result when training the model through extremely high number of epochs. In some cases the higher number of epochs will not affect the position of moving points much. This is case specific.} and due to this, the decision boundary changed drastically to yield incorrect results because the effects of randomness in are more pronounced. 

In order to address this flaw we can update the algorithm with an addendum using , such that after the movement vector $\Vec{b}$ is calculated, the algorithm checks the distance between both the migrating points used to draw the decision boundary and if the distance is less than\footnote{Less than is included because it is possible that the previous update caused the distance between the moving points to fall below the threshold.} or equal to a predefined threshold $\alpha$ (say) then the unit displacement vector is adjusted such that the point which is to be moved along $\Vec{b}$, moves along $\Vec{b}'$.  
\[
    f(y) = \begin{cases}
		\Vec{b}' = \Vec{b} - (\hat{EF} \times |\Vec{b}|\cos{\delta}) & \delta \leq 90\si{\degree}\\
		\Vec{b}' = \Vec{b} & \delta > 90\si{\degree} 
	\end{cases} 
\]
Where $\delta$ is the angle between $\Vec{b}$ and $\overrightarrow{FE}$ 
What this does is simply cancel out the displacement of the points towards each other if they come too close, judged by an arbitrary threshold $\alpha$. This is done only  if the angle between calculated displacement vector and $\overrightarrow{FE}$ is less than $90\si\degree$ because if the angle is greater than $90\si\degree$ the points would naturally move further away.

\subsection{Higher Dimensional Datasets} \label{sec:sec34}
We only discussed about a dataset that has two dimensions. The method to find decision boundary was also explicit to a line that can only separate points on a plane. What about datasets like:

\[
    D = 
    \begin{bmatrix}
    x_1^{[1]} & x_2^{[1]} & x_3^{[1]} & \dots & x_n^{[1]} \\
    x_1^{[2]} & x_2^{[2]} & x_3^{[2]} & \dots & x_n^{[2]} \\
    x_1^{[3]} & x_2^{[3]} & x_3^{[3]} & \dots & x_n^{[3]} \\
    \vdots & \vdots & \vdots & \ddots & \vdots \\
    x_1^{[m]} & x_2^{[m]} & x_3^{[m]} & \dots & x_n^{[m]} \\
    \end{bmatrix}
\]
where $m$ is the number of examples and $n$ is the number of features.

In order to deal with data points described in n-dimensions, we need to first establish a few points.

\begin{enumerate}
    \item If we have a dataset of n features, then that dataset can be mapped to an ambient space of n dimensions.
    \item It would be meaningless to define an hyperplane of m dimensions in q dimensions such that $m \geq q$. Therefore every $m$ dimension hyperplane is defined or "embedded"in $m+1$ or more dimensions.   
\end{enumerate}

Hence, we shall state: 
\label{stmt1}A hyperplane of $n-1$ dimensions can separate an ambient space of n dimensions into two regions of infinite extent \footnote{Just like a plane has an infinite area, 3-D space has infinite volume, n-dimensional space would have an infinite "extent" which does not have a specific quantity}(See Theorem 3 in Appendix A).

Therefore in order to accomplish the task of making a decision boundary, we just need to find the equation of this hyperplane of n-1 dimensions. The general equation of such a hyperplane is given by: 
\[
    a_1x_1 + a_2x_2 + \dots a_nx_n = c 
\]
where $c$ is a constant. This is analogous to equation (2). 

The distance $d$ between a point and a hyperplane in is given by:
\[
    d = \frac{|a_1x_1 + a_2x_2 + a_3x_3 ... + a_nx_n|}{\sqrt{a_1^2 + a_2^2 + a_3^2 ... + a_n^2 }}
\]

Now before we move to how we will calculate the equation of the hyperplane from a set of given points, we shall discuss the the nature of the points required to define a hyperplane of n-1 dimensions embedded in an ambient space of n dimensions: 
\begin{enumerate}
    \item To describe a hyperplane of n-1 dimensions, you need at least n points.
    \item The set of n points must not lie on the same n-2 hyperplane.
\end{enumerate}

Earlier we used an equation specific to lines to find out the equation of the line that would pass through two points and form a decision boundary. Now we shall look at the method that will give us the equation of any n-1 dimensional hyperplane embedded in an n-dimensional hyperspace using n points.

Let each point that will be used to draw the hyperplane be of the type $t = (x_1^{[1]}, x_2^{[1]}, \dots, x_n^{[1]})$. We need to have n such points to describe the hyperplane. To find the equation solve the determinant\footnote{The reasoning behind this method is explained well in this wonderful answer\cite{amd_answer_2018} by user amd on Mathematics Stack Exchange}:
\[
    det
    \begin{bmatrix}
    x_1 & x_2 & x_3 & \dots & x_n & 1 \\
    x_1^{[1]} & x_2^{[1]} & x_3^{[1]} & \dots & x_n^{[1]} & 1 \\
    x_1^{[2]} & x_2^{[2]} & x_3^{[2]} & \dots & x_n^{[2]} & 1 \\
    \vdots & \vdots & \vdots & \ddots & \vdots & \vdots \\
    x_1^{[m]} & x_2^{[m]} & x_3^{[m]} & \dots & x_n^{[m]} & 1\\
    \end{bmatrix} = 0
\]

Solving the determinant would give us an equation of the form $w_1x_1 + w_2x_2 + w_3x_3 \dots + w_nx_n + c = 0$ where c is a constant term (usually a sum of constants values). Comparing this equation to the general equation of hyperplane (equation (2)), we get $w_1 = a_1, w_2 = a_2, w_3 = a_3 \dots w_n = a_n$. This therefore givers us the coefficients of the hyperplane equation, using which we can apply the same algorithm as discussed in the case of 2-D datasets.  

\subsection{Performance Evaluation}
We evaluated the model against a Perceptron, a KNN model and a SVC model initialized using Google's sklearn library (\cite{pedregosa_scikit-learn_2011}). To get a generalized performance we calculated the mean errors over 500 synthetic datasets each containing two 2 feature columns and one true class label column  without changing the parameters. Both Perceptron and SVC class have adaptive learning rate feature, which is absent in our classifier implementation. 
The 500 datasets were obtained using make\_blobs method of sklearn. Integer values from 0-50 were used as random seeds and further the standard deviation of these 50 datasets was varied over values from 1.0 to 2.0 which gave us 10 datasets from each one randomly generated dataset. Hence a total of 500 datasets. The Perceptron was initialized with (random\_state = 8). The KNN model was initialized with (n\_neighbors = 3, n\_jobs = -1 ).

The results of the performance evaluation are as follows: 
\begin{table}[h!]
\centering
\begin{tabular}{|c|c|c|c|}
\hline
Algorithm   & Training Accuracy & Test Accuracy & Generalization Gap \\
\hline
KNN         & 0.9834            & 0.9705        & 0.0129             \\
\hline
SVM         & 0.9774            & 0.9735        & 0.0039             \\
\hline
MPA         & 0.9762            & 0.9760        & 0.0002             \\
\hline
Perceptron & 0.9581            & 0.9538        & 0.0043             \\            
\hline
\end{tabular}
\caption{Performance of 4 classifiers on 500 synthetic datasets.}\label{tabel1}
\end{table}

The similar errors of SVC and our classifier are expected since they are basically doing the same thing i.e creating a hyperplane. Only the methods used to create the hyperplanes differ.  

\begin{figure}[h]
    \centering
    \includegraphics[width=0.45\textwidth]{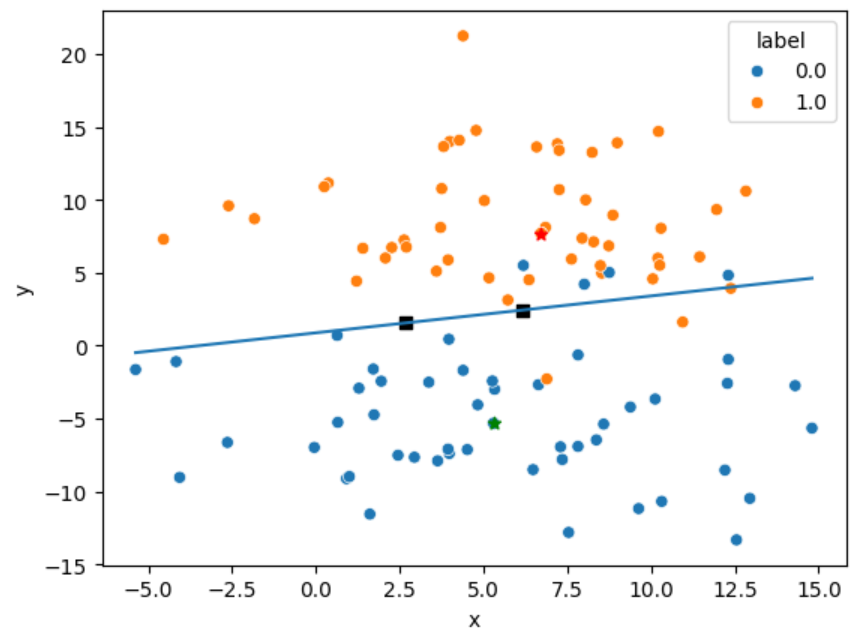}
    \caption{The decision boundary calculated by moving points algorithm on a synthetic dataset.}
    \label{fig:fit}
\end{figure}

Upon testing the custom made implementation of the algorithm on the the Iris dataset \cite{iris_53} to classify the samples of Iris-setosa and Iris-versicolor with 50 samples of each class, it was found that the algorithm was able to classify the linearly separable dataset in two dimensions with 100\% accuracy. The feature vectors selected are: SepalWidthCm and SepalLengthCm.

\begin{figure}[h]
    \centering
    \includegraphics[width=0.45\textwidth]{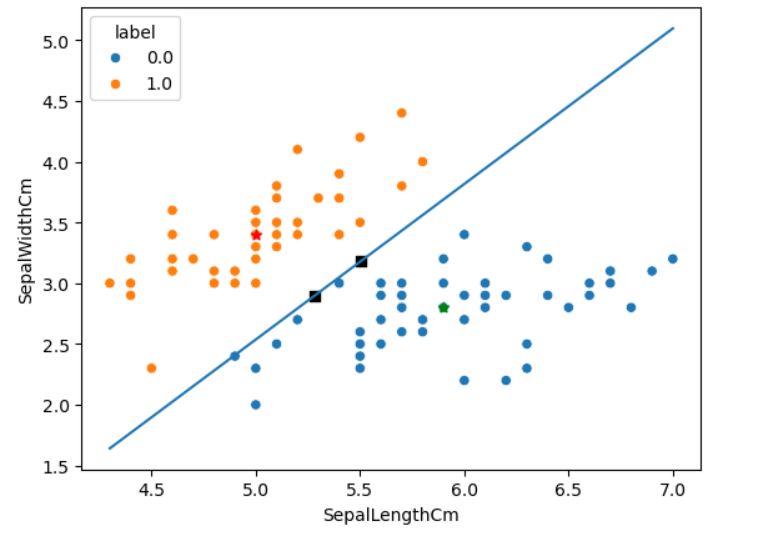}
    \caption{The decision boundary calculated by the algorithm on Iris dataset, label 0: Iris-versicolor, label 1: Iris-setosa.}
    \label{fig:setosaversicolor}
\end{figure}

To test the generalization abilities of the algorithm, it was tested on three standard datasets; Pima Indians diabetes dataset, Iris dataset, Penguins species dataset. 
The results of the model were compared to the results obtained with SVM using a linear kernel on training and testing dataset with test size of 0.2, and 3 features obtained using principal component analysis. The results were compared against SVM with linear kernel to understand how well the decision boundary found using this method is able to classify a given data. A linear SVM also finds a linear boundary in the same hyperspace so it serves as a good reference for analysis of our results. Since SVM is a proven algorithm, if the results are similar then we can be sure that the proposed algorithm works. 
The results\footnote{For training we used approach similar to Section 3.2 and not Section 3.4} are given in Table 2 (Appendix B).

\subsection{Outlook and Further Developments}
\subsubsection{Polynomial Interpolation}
Two popularly used interpolation techniques are Lagrange interpolation \cite{farmer_lagranges_2018} and Newton's divided difference interpolation formula \cite{das_newtons_2016}. However both of these techniques are susceptible to Runge's phenomenon, which refers to the inaccurate oscillations in the polynomial around the boundaries of the interval. The critical effect it can have on the predictions in machine learning can be seen from \cite{noauthor_explore_2018}. 

The Runge's phenomenon can be avoided if we use Chebyshev distribution to calculate the polynomial by concentrating several points of towards the end of the distribution but that would increase the training time and at the same time make the calculation of polynomial expensive. Another down side is that whatever amount of points we use to ensure that the function does not oscillate do not add to the information about underlying patterns in data. To understand what this means let us state an important mechanism about the working of this algorithm. The algorithm essentially has two parts:

\begin{itemize}
    \item The Points - they are the objects that contain the information about the underlying patterns.
    \item Method of Inference - basically the method that we use to draw or \textit{find} the decision boundary.  
\end{itemize}

For eg. when we defined the algorithm for strictly making a linear decision boundary, the steps for calculating the equation of that boundary (eq. 3) are the method of inference. Similarly, using more than the stipulated $n-1$ points to find the decision boundary of dataset of $n$ dimensions calls for a different method of inference from those points. Once the training ends, the amount of information contained in the points is fixed but how accurately and efficiently we are able to extract that information depends on the method of inference. Coming to the original statement that the increased number of points required to remove the effects of Runge's phenomenon do not add to the information about data pattern contained by the points. It solely focuses on improving the accuracy of method of inference. This can be seen trivially using the following example: 

\begin{figure}[htbp]
    \centering
    \includegraphics[width=0.45\textwidth]{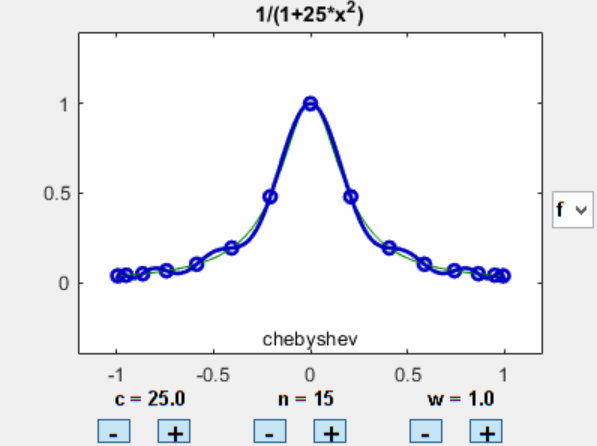}
    \caption{Credit: \cite{noauthor_explore_2018}}
    \label{fig:runge_err}
\end{figure}

15 points were used to fit the polynomial to a satisfactorily, but it can be seen by inspection that these many polynomials are not needed for such a simple graph. Hence, the extra points do not add to the information present in the model. 

\subsubsection{Splines}
This is an alternate method that is used in interpolation problems. It is basically a piece wise function made up of linear, quadratic, cubic or polynomials of any degree. The biggest plus point of this method is that it can avoid Runge's phenomenon even at higher degrees. Hence, in our use case we can make do with a relatively short number of points.


\newpage

\appendix
\section{}

\begin{theorem}
    A hyperplane construct of n - 1 dimensions can divide an ambient space of n dimensions into three disjoint sets $X_{1}$,  $X_{2}$, and $X_{3}$ (say) of points, with one set containing all and only points that satisfy the equation of the hyperplane.
\end{theorem}
\begin{proof}
By definition a hyperplane is a subspace whose dimensions is one less than the ambient space. Let $w$ be a normal vector normal to the hyperplane $H \in R^{n}$ and $x$ be any point inside the ambient space. Therefore,
\[
    w^{T}.x + b = 0
\]
is the general equation of the hyperplane, where $w = [w_{1}, w_{2}, w_{3}\cdots, w_{n}]$. Now, we can define three sets of points from this equation. Points that satisfy;
\[
    w^{T}.x + b = 0
\]
or,
\[
    w^{T}.x + b < 0
\]
or,
\[
    w^{T}.x + b > 0
\]
No point can satisfy two or more of these equations simultaneously and every point in the ambient space would satisfy one of these equations hence our theorem is true.
\end{proof}

\subsection{Remarks on testing with 3 features}

For three-dimensional data, we can implement code based on the approach proposed in \ref{sec:sec34} but instead, we use an approach similar to \ref{sec:sec32}. Based on \ref{stmt1}, we can conclude that a 2-D plane would be able to separate a 3-D space into 3 disjoint sets. 

The equation of such plane is given by:
\[
    ax + by + cz + d = 0
\]
The displacement, D is given by:
\[
    D = \frac{ax_{n} + by_{n} + cz_{n} + d}{\sqrt{a^2 + b^2 + c^2}}
\]
For calculating the coefficients, utilities from SymPy \cite{10.7717/peerj-cs.103} were used to create a plane from three random points, find equation of that plane and perform the necessary calculations to implement the learning model.  
The rest of the method is the same. 

\subsection{Explanation of decision boundary movement approach}
For the explanation, we will consider a case with three points $P$, $Q$ and $R$ on the cartesian plane such that $P \subset X$ and $\{Q, R\} \subset Y$. Let $l$ be a line on the same plane such that it divides the plane into two sets $S_1$ and $S_2$, with equation $Ax + By + C = 0$. For the purpose of this explanation we will ignore the points on the line $l$. We will only consider case where $Q$ is in the region between $P$ and $R$. The displacement between line and a point $E$ with coordinates $(x_E, y_E)$is given by:   
\begin{equation}
    \lambda(E) = \frac{Ax_E + By_E + C}{\sqrt{A^2 + B^2}}
\end{equation}
Whether two points $G$ and $H$ are in the same region of plane divided by the line is determined by the sign of $\lambda(G)$ and $\lambda(H)$. If they have the same sign then they are in the same subset ($S_1$ or $S_2$). 

Suppose that point $l$ is drawn in such a way that $\{P, Q\} \subset S_1$ and $R \subset S_2$. We want to arrive at a condition such that $P \subset S_1$ and $\{Q, R\} \subset S_2$ so that $l$ can be considered as a decision boundary.  

Since $\{P, Q\} \subset S_1$ we can conclude that the sign of $\lambda(P)$ and $\lambda(Q)$ is the same and the sign of $\lambda(R)$ is the opposite of $\lambda(R)$. In order to correct this, the sign of $\lambda(Q)$ must be reversed. To do this, the line $l$ should move close to $Q$, bringing $\lambda(Q)$ to 0 and then away from it to give it the opposite sign. Since $P$ lies on the same side as $Q$ but further away from $l$, reducing $\lambda(Q)$ would also reduce $\lambda(P)$. Hence, as proposed in Section 3.2, by moving the decision boundary towards the opposite class (represented by its mean), we can move the decision boundary in such a way that it corrects a misclassification.

\section{}

The following table demonstrates training and testing accuracies of moving points algorithm and support vector machines on the same datasets with the same training and testing subsets in each run. 
\begin{itemize}
    \item Iris dataset: 150 samples \cite{iris_53}
    \item Palmer Penguins dataset, 344 samples \cite{penguins_data}
    \item Pima Indians Diabetes Dataset, 768 samples \cite{diabetes_data} 
\end{itemize}
We reduced the dimensionality of the datasets to 3 using PCA from sklearn (\cite{pedregosa_scikit-learn_2011}). 

For SVM, the kernel was set to "linear", other hyperparameters were not tuned, implementation of SVM was taken from sklearn(version 1.5.2). For the proposed algorithm, learning rate was varied between 0.00003 to 0.00008 depending on the dataset for the most optimum performance. All datasets were standardized before training but after splitting into training and testing datasets to prevent data leakage.  

Testing machine details:
\begin{itemize}
    \item Operating System: Ubuntu 24.04.1 LTS
    \item OS Type: 64-bit
    \item Processor: 12th Gen Intel® Core™ i5-12450H × 12
    \item Environment: Jupyter notebook
    \item Language and version: Python 3.12.6
\end{itemize} 

For code, documentation and access to necessary github repositories, contact the author by email. 
\begin{table}[]
\centering
\begin{tabular}{|c|c|c|c|c|c|}
\hline
S.NO    & \multicolumn{2}{c|}{SVM} & \multicolumn{2}{c|}{MPA} & DATASET                                                                                     \\
\hline
        & Training    & Testing   & Training    & Testing  & \multirow{7}{*}{\begin{tabular}[c]{@{}c@{}}PIMA INDIANS\\ DIABETES\\ DATASET\end{tabular}} \\
1.      & 72.63       & 74.67     & 70.35       & 77.27     &                                                                                            \\
2.      & 72.80       & 73.37     & 71.49       & 74.07     &                                                                                            \\
3.      & 72.47       & 67.53     & 73.45       & 67.53     &                                                                                            \\
4.      & 70.84       & 80.51     & 71.63       & 74.02     &                                                                                            \\
5.      & 72.14       & 73.37     & 71.82       & 67.53     &                                                                                            \\
\hline
AVERAGE & 72.17       & 73.37     & 71.82       & 72.08     &                                                                                            \\
\hline
1.      & 96.59       & 95.45     & 94.95       & 95.45     & \multirow{6}{*}{\begin{tabular}[c]{@{}c@{}}PENGUIN \\ DATASET\end{tabular}}                \\
2.      & 96.02       & 93.18     & 95.87       & 84.09     &                                                                                            \\
3.      & 96.02       & 97.72     & 94.49       & 97.72     &                                                                                            \\
4.      & 95.45       & 100.00    & 100.00      & 97.72     &                                                                                            \\
5.      & 95.45       & 97.72     & 95.41       & 95.45     &                                                                                            \\
\hline
AVERAGE & 95.90       & 96.81     & 96.14       & 94.08     &                                                                                            \\
\hline
1.      & 97.5        & 85        & 92.5        & 85        & \multirow{6}{*}{\begin{tabular}[c]{@{}c@{}}IRIS \\ DATASET\end{tabular}}                  \\
2.      & 95          & 95        & 95          & 92.5      &                                                                                            \\
3.      & 92.5        & 100       & 88.75       & 95        &                                                                                            \\
4.      & 97.5        & 85        & 90          & 90        &                                                                                            \\
5.      & 96.25       & 90        & 90          & 90        &                                                                                            \\
\hline
AVERAGE & 95.75       & 91        & 91.25       & 90.5      &              \\    \hline                                                                       

\end{tabular}
\caption{Testing of moving points algorithm and comparison with performance of SVM.}\label{table2}
\end{table}

\newpage
\vskip 0.2in
\bibliography{references}

\end{document}